\documentclass[preprint,12pt]{elsarticle}

\usepackage{lscape}
\usepackage{mathtools}

\usepackage{longtable}
\usepackage{subcaption}
\usepackage{multirow}
\usepackage{hyperref}
\usepackage{amssymb}
\usepackage{physics}
\usepackage{mathrsfs}
\usepackage{amsmath}
\usepackage{amsthm}
\newtheorem{thm}{Theorem}
\newtheorem{lemma}{Lemma}
\usepackage{xcolor}
\theoremstyle{definition}
\newtheorem{definition}{Definition}[section]

\journal{Elsevier}

\begin{document}
\begin{frontmatter}



\title{LSTSVR-PI: Least square twin support vector regression with privileged information}


\author[inst1]{Anuradha Kumari}

\affiliation[inst1]{organization={Department of Mathematics, Indian Institute of Technology Indore},
            addressline={Simrol}, 
            city={Indore},
            postcode={453552}, 
            state={Madhya Pradesh},
            country={India}}

\author[inst1]{M. Tanveer\corref{mycorrespondingauthor}}


\cortext[mycorrespondingauthor]{Corresponding author}
\begin{abstract}
In an educational setting, a teacher plays a crucial role in various classroom teaching patterns. Similarly, mirroring this aspect of human learning, the learning using privileged information (LUPI) paradigm introduces additional information to instruct learning models during the training stage. A different approach to train the twin variant of the regression model is provided by the new least square twin support vector regression using privileged information (LSTSVR-PI), which integrates the LUPI paradigm to utilise additional sources of information into the least square twin support vector regression. The proposed LSTSVR-PI solves system of linear equations which adds up to the efficiency of the model. Further, we also establish a generalization error bound based on the Rademacher complexity of the proposed model and incorporate the structural risk minimization principle. The proposed LSTSVR-PI fills the gap between the contemporary paradigm of LUPI and classical LSTSVR. Further, to assess the performance of the proposed model, we conduct numerical experiments along with the baseline models across various artificially generated and real-world datasets. The various experiments and statistical analysis infer the superiority of the proposed model. Moreover, as an application, we conduct experiments on time series datasets, which results in the superiority of the proposed LSTSVR-PI. \\

\end{abstract}




\end{frontmatter}


\section{Introduction}
Support vector machine (SVM) \cite{vapnik1999overview} stands out as a powerful machine-learning algorithm, renowned for its robust mathematical foundation. It finds applications across diverse domains, including cybersecurity \cite{deylami2012adaboost}, remote sensing \cite{chen2015robust}, medical imaging \cite{liu2018robust}, sediment analysis \cite{hazarika2020modeling}, and beyond. Initially designed for classification tasks, SVMs have evolved to address regression problems, leading to the development of support vector regression (SVR) \cite{smola2004tutorial}. SVR serves the purpose of identifying optimal regressors for regression problems, incorporating two sets of constraints that effectively partition the data samples within the $\epsilon$-insensitive region \cite{drucker1996support}. SVR finds extensive applications in real-world scenarios, including cement industries \cite{pani2014soft}, travel-time prediction \cite{wu2004travel}, location estimation \cite{wu2007location}, and various other domains. 

Despite the strong generalization performance and strong mathematical foundation with a unique global solution, SVM solves a quadratic programming problem (QPP) which has a high time complexity of $O(m^3)$, $m$ corresponds to the number of training samples. To address this, a classifier featuring non-parallel hyperplanes known as twin support vector machine (TSVM) \cite{khemchandani2007twin} is introduced. TSVM forms separate hyperplanes for each class and solves two QPPs resulting in approximately four times more efficiency than SVM \cite{khemchandani2007twin}. The comprehensive review by \citet{tanveer2022comprehensive} delves into different variants of TSVM. \par
Motivated by the concept of TSVM, \citet{peng2010tsvr} introduced the twin support vector machine for regression (TSVR). TSVR, akin to TSVM, employs non-parallel hyperplanes and establishes $\epsilon$-insensitive lower and upper bound functions by addressing two QPPs, with each QPP being smaller than that of SVR. Further, different variants of TSVR are introduced such as ($\epsilon$-TSVR) \cite{shao2013varepsilon}, weighted TSVR \cite{xu2012weighted}, twin projection SVR \cite{peng2014twin} and so on. The aforementioned variants solve a pair of QPPs which has high time and memory complexity. To overcome the high computational burden twin least square SVR (TLSSVR) \cite{zhao2013twin} is introduced. TLSSVR replaces the inequality constraints of TSVR with equality constraints, leading to the solution of a system of linear equations to train the model. Further different least square variants of TSVR are proposed such as primal least square TSVR (PLSTSVR) \cite{huang2013primal}, which solves the system of equations in the primal space, LSTSVR and sparse LSTSVR (FS-LSTSVR) \cite{huang2016sparse}, which results in a sparse solution of LSTSVR.  \par   


\citet{bi2003geometric} discussed the geometric approach of SVR by viewing regression problems as classification problems aimed at distinguishing between the overestimates and underestimates of regression data. Further, \citet{khemchandani2016twsvr} extended the idea of \cite{bi2003geometric} for twin variants and proposed regression via TSVM (TWSVR). In \cite{khemchandani2016twsvr}, authors claim that TSVR \cite{peng2010tsvr} is not in the true spirit of TSVM. In each of the twin optimization problems, TSVR either addresses the upbound or downbound regressor, whereas, the TWSVR model adheres to TSVM  principles, where the objective of each QPP has up (down) bound regressor with constraints of down (up) bound regressor.      
Recently, \citet{huang2022overview} provided an overview of TSVR, which comprehensively covers different variants of the twin regression approach. Further, to address the heterogeneity and asymmetry in the training data twin support vector quantile regression (TSVQR) \cite{ye2023twin} is introduced. TSVQR efficiently represents the heterogeneous distribution information with regard to all components of the data points using a quantile parameter. Along with the rapid growth in the variants of TSVR, it has wide applications in prediction such as wind speed \cite{hazarika2022wavelet}, brain age \cite{ganaiea2022brain}, traffic flow \cite{nidhi2022traffic} and so forth. \par

 \citet{vapnik2009new} introduced a novel learning paradigm known as learning using privileged information (LUPI), to improve the generalization performance of learning models. In the conventional supervised learning paradigm, only the training data is available in the training stage. However, LUPI deviates from this norm by providing additional information for the training data during the training stage.  This supplementary information is referred to as privileged information. Importantly, this privileged information is not available during the test stage. There is an analogy between LUPI and the human learning process. In a classroom, a teacher can impart crucial and helpful information about the course to students, facilitating better knowledge acquisition. Similarly, in the LUPI paradigm, models can leverage additional information during training, akin to the role of a teacher, ultimately leading to improved generalization performance compared to traditional learning models. The utilization of PI extends to diverse fields including data clustering \cite{feyereisl2012privileged}, visual recognition \cite{motiian2016information}, multi-view learning \cite{tang2019coupling}, face verification \cite{xu2015distance} and so on. In the context of person re-identification, PI is employed \cite{yang2017person} to establish a locally adaptive decision rule. \par
The integration of PI with various machine-learning techniques has led to notable advancement. 
One such approach is SVM+ introduced by \citet{vapnik2009new}, which formulates an SVM optimization problem within the LUPI setting. However, SVM+ employs the $l_2$ norm, doubling the number of parameters and resulting in extended time for parameter tuning. In efforts to enhance computational efficiency, \citet{niu2012nonlinear} proposed SVM+ with the $l_1$ norm as an alternative to $l_2$. This modification streamlines the optimization process and reduces computational complexity. While both variants of SVM+ with $l_1$ and $l_2$ norms exhibit good generalization performance, it's important to note that their performance can be impacted in the presence of noise in the data, potentially leading to suboptimal results. To address these challenges a robust SVM+ model (R-SVM+) is introduced by \citet{li2018r}. Additionally, \citet{che2021twin} integrated PI with TSVM and introduced TSVM-PI which demonstrates improved computational efficiency and employs two non-parallel hyperplanes, presenting itself as a compelling alternative to SVM+. The utilization of PI in the training process has further enhanced the performance of various classification models such as kernel ridge regression-based auto-encoder for one-class classification (AEKOC) \cite{gautam2017construction}, kernel ridge regression-based one-class classifier (KOC) \cite{vovk2013kernel}, random vector functional link neural network (RVFL) \cite{malik2023random}, and introduced AEKOC+ \cite{gautam2020aekoc+}, KOC+ \cite{gautam2019koc+} and RVFL+ \cite{zhang2020new}, respectively. \par

  Taking inspiration from the improved performance of models utilizing PI and non-parallel hyperplanes constructed in TSVM-PI for classification, in this paper, we propose least square twin support vector regression with privileged information (LSTSVR-PI).  The main highlights of the proposed algorithm can be listed as follows:
\begin{itemize}
    \item The proposed LSTSVR-PI integrates PI into LSTSVR and provides extra information for training along with enhancing the generalization performance of the model. This represents the initial application of PI within the framework of the least square variant of a non-parallel hyperplane regressor.
    \item We employ the regularization terms in the proposed optimization problems corresponding to both the regressor function and correcting function that adheres to the SRM principle and prevents overfitting.  
    \item The proposed LSTSVR-PI solves a system of linear equations, leading to an efficient model with less computational burden.  
    \item The numerical experimental results and statistical analysis carried out over artificially generated datasets and real-world datasets demonstrate the superiority of the proposed model. The proposed model has a real-world application on time-series datasets. 
\end{itemize}
 
The rest of the paper is organized as follows: Section \ref{sec:related_work} briefly discusses the related work. The linear and non-linear case of the proposed LSTSVR-PI is thoroughly discussed in Section \ref{sec:proposed work}.  Numerical experiments are analyzed in Section \ref{sec:exp}. The application of the proposed LSTSVR-PI is discussed in Section \ref{sec:application}. Section \ref{sec:conclusion} concludes the paper with future work.

\section{Related work}{\label{sec:related_work}}
In this section, we briefly discuss the existing models, specifically SVR, SVR+ and TSVR, which are in close relation with the proposed LSTSVR-PI. We mainly discuss the non-linear case of the models. Let S$=\{(x_1,x_1^*,y_1), (x_2,x_2^*,y_2),$\\
$\ldots,(x_m,x_m^*,y_m)\}$ denote the training samples, where $x_i \in \mathbb{R}^n$, $x_i^* \in \mathbb{R}^n$  and $y \in \mathbb{R}$. For the $i^{th}$ data point $x_i$ corresponds to the training features, $x_i^*$ corresponds to the privileged information and $y_i$ corresponds to the target value. Assume $M$ and $M^*$ denote matrices containing training input data and their privileged information, respectively. $Y$ denotes the vector containing the target values of input data. Assume $\mathscr{\phi}$ be the nonlinear map from input space $(\mathbb{R}^n)$ to higher dimensional space, and  $K(x_i,x_j^T)=\phi(x_i) \cdot \phi(x_j^T)$ corresponds to the kernel matrix. 

\subsection{Support vector regression \cite{smola2004tutorial}}
The classical SVR employs an $\epsilon$ -insensitive loss function with the objective of discovering a regressor function that allows a maximum deviation of $\epsilon$  from the actual target values of the training data. In the non-linear case, the objective is to find the flattest function within the feature space, rather than in the input space. The equation of the regressor function for the non-linear case is given as follows:
\begin{align} \label{eq:decision}
    r(x)=u^T\phi(x)+b,
\end{align}
 where $u$ and $b$ are the parameters of the regression function. The optimization problem for non-linear SVR is expressed as:
 \begin{align}\label{qpp:SVR}
     & \underset{u, \zeta_1,\zeta_2}{\text{min }} \frac{1}{2}\norm{u}^2+c_1(e^T\zeta_1+e^T \zeta_2) \nonumber \\
     & \text{such that} ~Y-(\phi(M)u+be) \leq \epsilon e+\zeta_1, ~~\zeta_1 \geq 0, \\
     & ~~~~~~~~~~~~~ (\phi(M)u+be)-Y \leq \epsilon e+\zeta_2,~~\zeta_2 \geq 0. \nonumber 
 \end{align}
Here, $c_1>0$ is the trade-off parameter between the penalty term and flatness of regressor. $\zeta_1$ and $\zeta_2$ are the vectors of slack variables. $e$ denotes a vector of ones having a suitable dimension. By employing the equation (\ref{eq:decision}) with the unknowns obtained from QPP (\ref{qpp:SVR}), the target value of the unseen samples can be determined. 

\subsection{SVR+  \cite{vapnik2009new}}
Utilizing PI, \citet{vapnik2009new} proposed SVM+ regression (SVR+), where the available PI is incorporated into SVR by using the correction function for slack variables. The model is trained using both the original feature as well as additional information from the training data samples. The two correcting functions of the SVR+ are given by: 
\begin{align*}
    r_1^*(x^*)={u_1^*}^T \phi(x^*)+b_1^*   \text{ and }   r_2^*(x^*)={u_2^*}^T \phi(x^*)+b_2^*.
\end{align*}
Here, $r_1^*(x)$ and $r_2^*(x)$ are the correcting functions for the approximation of slack variables $\zeta_1$ and $\zeta_2$, respectively. The regressor function is given by equation (\ref{eq:decision}). The optimization problem of SVR+ is expressed as:
\begin{align}
    \underset{u, u_1^*,u_2^*,b,b_1^*,b_2^*}{\text{min }}&\frac{1}{2}[\norm{u}^2+c_1(\norm{u_1^*}^2+\norm{u_2^*}^2)]+c_2[({u_1^*}^T\phi(M^*)+b_1^*e)+({u_2^*}^T\phi(M^*)+b_2^*e)] \nonumber \\
    \text{such that } & Y-(\phi(M)u+be) \leq \epsilon e+( \phi(M^*){u_1^*}+b_1^*e), \nonumber \\
     & (\phi(M)u+be)-Y \leq \epsilon e+( \phi(M^*){u_2^*}+b_2^*e), \\
     &( \phi(M^*){u_1^*}+b_1^*e) \geq 0, \nonumber \\
    &( \phi(M^*){u_2^*}+b_2^*e) \geq 0. \nonumber
\end{align}
Here $e$ signifies a column vector of ones with appropriate dimensions. $c_1$, $c_2$ corresponds to the positive regularization terms. One can obtain the unknowns $u,u_1^*,u_2^*,b,b_1^*,b_2^*$ using the dual problem, explained in \cite{vapnik2009new}. 

\subsection{TSVR \cite{peng2010tsvr}}
TSVR constructs non-parallel hyperplanes akin to TSVM, but they constitute distinct geometric and structural properties. Firstly, while TSVM aims to find hyperplanes close to each class, TSVR identifies $\epsilon$-insensitive down and upbound regressor functions. Secondly, in TSVM, the objective function of each QPP relies solely on the data points from one class, whereas TSVR incorporates all data points from both classes into its QPPs. These distinctions highlight the unique characteristics and objectives of each approach. The $\epsilon_1$- insensitive down bound regressor and $\epsilon_2$-insensitive up bound regressor for non-linear cases are given by:
\begin{align}\label{eq:secision_TSVR}
r_1(x)=K(x^T,M^T)u_1+b_1 \text{ and } r_2(x)=K(x^T,M^T)u_2+b_2, 
\end{align}
respectively. The optimization problem for TSVR is written as:
\begin{align}
    \underset{u_1,b_1,\zeta_1}{\text{min }}& \frac{1}{2}\norm{Y-e \epsilon_1-(K(M,M^T)u_1+eb_1)}^2 +c_1e^T\zeta_1  \nonumber \\
    \text{such that } & Y-(K(M,M^T)u_1+eb_1) \geq e\epsilon_1-\zeta_1,~~\zeta_1 \geq 0,
\end{align}
and
\begin{align}
    \underset{u_2,b_2,\zeta_2}{\text{min }}& \frac{1}{2} \norm{Y+e \epsilon_2-(K(M,M^T)u_2+eb_2)}^2 +c_2e^T\zeta_2  \nonumber \\
    \text{such that } & (K(M,M^T)u_2+eb_2)-Y \geq e\epsilon_2-\zeta_2,~~\zeta_2 \geq 0,
\end{align}
where $\zeta_1$ and $\zeta_2$ are the slack vectors; $c_1$, $c_2$ are positive regularization terms. 
Thus, by solving two smaller-sized QPPs, TSVR demonstrates greater computational efficiency compared to SVR.

For an unseen data point, the target is given by the function 
\begin{align} \label{eq:dectsvr}
    r(x)=\frac{1}{2}(r_1(x)+r_2(x))=\frac{1}{2}(u_1+u_2)^TK(M,x)+\frac{1}{2}(b_1+b_2).
\end{align}


\section{Proposed Work}{\label{sec:proposed work}}
The conventional LSTSVR relies solely on the training features for learning. However, learning in the presence of PI enhances the generalization performance of the model. The LUPI paradigm leverages additional information during the training phase. Building upon this concept and integrating PI into LSTSVR, in this section, we propose least square twin support vector regression with privileged information (LSTSVR-PI). We thoroughly discuss the linear and non-linear variants of the proposed LSTSVR-PI. In accordance with \cite{vapnik2009new}, the slack variable is formed as the correcting function for incorporating PI into the proposed optimization problems.

\subsection{Proposed LSTSVR-PI: linear case}
In this subsection, we discuss the linear case of the proposed LSTSVR-PI. The equations of two $\epsilon$-insensitive proximal linear functions are given as:
\begin{align}
    r_1(x)=u_1^Tx+b_1 \text{ and } r_2(x)=u_2^Tx+b_2, 
\end{align}
and correction functions of each linear function are given as: 
\begin{align}
    p_1(x^*)={u_1^*}^Tx^*+b_1^* \text{ and } p_2(x^*)={u_2^*}^Tx^*+b_2^*,
\end{align}
respectively. Here, $u_i$, $u_i^* \in \mathbb{R}^n$ and $b_i,~b_i^* \in \mathbb{R}$ for $(i=1,~2)$. The optimization problems for the linear case can be expressed as:
\begin{align}{\label{eq:primalpropo1}}
    \underset{u_1,b_1,u_1^*, b_1^*}{\text{min }}& \frac{c_1}{2} (\norm{u_1}^2+b_1^2)+ \frac{c_2}{2}(\norm{u_1^*}^2+{b_1^*}^2) \nonumber\\ 
    &+ \frac{1}{2} \norm{Y-(Mu_1+eb_1)}^2+c_3e^T(M^*u_1^*+eb_1^*) \nonumber \\
    \text{such that } & Y-(Mu_1+eb_1) = -\epsilon_1 e-(M^*u_1^*+eb_1^*)   
\end{align}
and 
\begin{align}{\label{eq:primalpropo2}}
    \underset{u_2,b_2,u_2^*, b_2^*}{\text{min }}& \frac{c_4}{2} (u_2^Tu_2+b_2^2)+ \frac{c_5}{2}({u_2^*}^Tu_2^*+{b_2^*}^2) \nonumber\\ 
    &+  \frac{1}{2}\norm{((Mu_2+eb_2)-Y)}^2+c_6e^T(M^*u_2^*+eb_2^*) \nonumber \\
    \text{such that } & (Mu_2+eb_2)-Y = -\epsilon_2 e-(M^*u_2^*+eb_2^*).
\end{align}
Here, $c_i$ $(i=1,2,3,4,5,6)$ corresponds to the positive regularization parameters. $e$ signifies a column vector of the appropriate dimension. The first term in the optimization problem (\ref{eq:primalpropo1}) is the regularization term corresponding to the first proximal linear function; the second term leads to the regularization term corresponding to the correcting function; the third term signifies the sum of the squared distance between the estimated function and the actual targets. The fourth term corresponds to the correcting function values using the PI of data. The regularization terms in the optimization problem (\ref{eq:primalpropo1}) result in the flatness of the function $r_1(x)$ and $p_1(x^*)$. The significance of the terms in the optimization problem (\ref{eq:primalpropo2}) can be explained similarly. 

In order to solve the problems (\ref{eq:primalpropo1}) and (\ref{eq:primalpropo2}), convert them to their dual forms using the Lagrangian function. For problem  (\ref{eq:primalpropo1}), it can be expressed as:
\begin{align}{\label{eq:Lagrangian1}}
    \mathscr{L}=&\frac{c_1}{2}(\norm{u_1}^2+b_1^2)+ \frac{c_2}{2}(\norm{u_1^*}^2+{b_1^*}^2)+ \frac{1}{2}{\norm{Y-(Mu_1+eb_1)}}^2 \nonumber \\
    &+c_3e^T(M^*u_1^*+eb_1^*)+\alpha^T(Y-(Mu_1+eb_1)+\epsilon e+(M^*u_1^*+eb_1^*)),
\end{align}
where $\alpha$ corresponds to the Lagrangian multipliers. Using the Karush Kuhn Tucker (K.K.T.) conditions, differentiating equation (\ref{eq:Lagrangian1}) with respect to $u_1,~b_1,~u_1^*,~b_1^*$, we get:
\begin{align}
    \frac{\partial \mathscr{L}}{\partial u_1}=&c_1 u_1-M^T(Y-(Mu_1+eb_1))-M^T \alpha=0, \label{eq:lag1}\\
    \frac{\partial \mathscr{L}}{\partial b_1}=&c_1b_1-e^T(Y-(Mu_1+eb_1))-e^T \alpha=0, \label{eq:lag2} \\
    \frac{\partial \mathscr{L}}{\partial u_1^*}=&c_2 u_1^*+c_3 {M^*}^Te +{M^*}^T\alpha=0 \label{eq:lag3}\\
    \frac{\partial \mathscr{L}}{\partial b_1^*}=&c_2 b_1^*+c_3 {e}^Te -{e}^T\alpha=0 \label{eq:lag4} 
\end{align}
Combining equations (\ref{eq:lag1}) and (\ref{eq:lag2}), we obtain
\begin{align}{\label{eq:pro1}}
    c_1v_1-G^T(Y-Gv_1)-G^T \alpha=0,
\end{align}
where $v_1=
\begin{bmatrix}
    u_1 \\
    b_1
\end{bmatrix}$ and $G=
 \begin{bmatrix}
     M  & e
 \end{bmatrix}$.
Similarly, combining equations (\ref{eq:lag3}) and (\ref{eq:lag4}), we get
\begin{align}{\label{eq:pro2}}
    c_2v_1^*+c_3{G^*}^Te+{G^*}^T\alpha=0,
\end{align}
where $v_1^*=\begin{bmatrix}
    u_1^* \\
    b_1^*
\end{bmatrix}$ and $G^*=\begin{bmatrix}
    M^* & e
\end{bmatrix}$. From the constraint of the optimization problem (\ref{eq:primalpropo1}), we have 
\begin{align}{\label{eq:pro3}}
    Y-Gv_1+\epsilon e+G^*v_1^*=0.
\end{align}
Solving equations (\ref{eq:pro1}), (\ref{eq:pro2}) and (\ref{eq:pro3}), we obtain\\
\begin{align}{\label{eq:determine_alpha}}
    \alpha=&(GG^T+\frac{c_1}{c_2}G^*{G^*}^T+\frac{1}{c_2}GG^TG^*{G^*}^T)^{-1}[c_1Y+c_1 \epsilon_1 e \nonumber \\
    & -\frac{c_1c_3}{c_2}G^*{G^*}^Te+GG^T\epsilon_1 e-\frac{c_3}{c_2}GG^TG^*{G^*}^Te ]
\end{align}
Following the same procedure as above for the optimization problem (\ref{eq:primalpropo2}) leads to the following equations: 
\begin{align}
    c_4v_2+G^T(Gv_2-Y)+G^T\beta&=0, \label{eq:prob21}\\
    c_5v_2^*+c_6{G^*}^Te+{G^*}^T\beta&=0,\label{eq:prob22}\\
    Gv_2-Y+\epsilon_2e+G^*v_2^*&=0, \label{eq:prob23}
\end{align}
where $\beta$ corresponds to the Lagrangian multiplier. From aforementioned equations (\ref{eq:prob21})-(\ref{eq:prob23}), we get
\begin{align}{\label{eq:determine_beta}}
    \beta=&(GG^T+\frac{c_4}{c_5}G^*{G^*}^T+\frac{1}{c_5}GG^TG^*{G^*}^T)^{-1}(-c_4Y+c_4 \epsilon_2 e \nonumber \\
    &-\frac{c_4c_6}{c_5}G^*{G^*}^Te+GG^T\epsilon_2 e-\frac{c_6}{c_5}GG^TG^*{G^*}^Te).
\end{align}
After determining $\alpha$, $\beta$ from equations (\ref{eq:determine_alpha}) and (\ref{eq:determine_beta}), respectively, the parameters of the hyperplanes $r_1(x)$ and $r_2(x)$ obtained using equations (\ref{eq:pro1}) and (\ref{eq:prob21}) are as follows: 
\begin{align}
    v_1&=(G^TG+c_1I)^{-1}G^T(Y+\alpha),\\
    v_2&=(G^TG+c_4I)^{-1}G^T(Y-\beta),
\end{align}
respectively.
The target value corresponding to an unknown sample is given by the function:
\begin{align*}
      r(x)=\frac{1}{2}(r_1(x)+r_2(x))=\frac{1}{2}x^T(u_1+u_2)+\frac{1}{2}(b_1+b_2).
\end{align*}

\subsection{Proposed LSTSVR-PI: Non-Linear case}
The presence of inseparable data points, which can occur when linear separation is not feasible, can be addressed using the proposed non-linear LSTSVR-PI.  
The $\epsilon_1$-insensitive down bound regressor and $\epsilon_2$-insensitive up bound regressor for non-linear cases are given by: 
\begin{align}\label{eq:decision_NTSVR}
\tilde{r}_1(x)=K(x^T,M^T)u_1+b_1 \text{ and } \tilde{r}_2(x)=K(x^T,M^T)u_2+b_2, 
\end{align}
respectively. The non-linear correcting function for each hyperplane is given as:
\begin{align}\label{eq:correcting_NTSVR}
    \tilde{p}_1(x^*)=K({x^*}^T,{M^*}^T){u_1^*}+b_1^* \text{ and }  \tilde{p}_2(x^*)=K({x^*}^T,{M^*}^T){u_2^*}+b_2^*,
\end{align}
respectively. As discussed for the linear case, the optimization problems for the proposed LSTSVR-PI are written as:
   \begin{align}{\label{eq:primal_nonlpropo1}}
    \underset{u_1,b_1,u_1^*, b_1^*}{\text{min }}& \frac{c_1}{2} (\norm{u_1}^2+b_1^2)+ \frac{c_2}{2}(\norm{u_1^*}^2+{b_1^*}^2) \nonumber\\ 
    &+ \frac{1}{2}\norm{Y-(K(M,M^T)u_1+eb_1)}^2
    +c_3e^T(K(M^*,{M^*}^T)u_1^*+eb_1^*) \nonumber \\
    \text{such that } & Y-(K(M,M^T)u_1+eb_1) = -\epsilon_1 e-(K(M^*,{M^*}^T)u_1^*+eb_1^*), 
\end{align}
and 
\begin{align}{\label{eq:primal_nonlpropo2}}
    \underset{u_2,b_2,u_2^*, b_2^*}{\text{min }}& \frac{c_4}{2} (\norm{u_2}^2+b_2^2)+ \frac{c_5}{2}(\norm{u_2^*}^2+{b_2^*}^2) +  \frac{1}{2}\norm{(K(M,M^T)u_2+eb_2)-Y} \nonumber \\
    &+c_6e^T(K(M^*,{M^*}^T)u_2^*+eb_2^*) \nonumber \\
    \text{such that } & (K(M,M^T)u_2+eb_2)-Y = -\epsilon_2 e-(K(M^*,{M^*}^T)u_2^*+eb_2^*).
\end{align} 
The significance of the terms in the primal problems (\ref{eq:primal_nonlpropo1}) and (\ref{eq:primal_nonlpropo2}) follows in a similar way as in the linear case. 
The Lagrangian function for problem (\ref{eq:primal_nonlpropo1}) is as follows:
\begin{align}{\label{eq:Lagrangian}}
    \mathscr{L}=&\frac{c_1}{2}(\norm{u_1}^2+b_1^2)+ \frac{c_2}{2}(\norm{u_1^*}^2+{b_1^*}^2)+ \frac{1}{2}{\norm{Y-(K(M,M^T)u_1+eb_1)}}^2 \nonumber \\
    &+c_3e^T(K(M^*,{M^*}^T)u_1^*+eb_1^*)+\gamma^T(Y-(K(M,M^T)u_1+eb_1)+\epsilon_1 e\nonumber \\
    &+K(M^*,{M^*}^T)u_1^*+eb_1^*),
\end{align}
where $\gamma$ is the Lagrangian multiplier. Moving similar to the linear case, we obtain 
\begin{align}
 c_1v_1-G^T(Y-Gv_1)-G^T \gamma&=0,\label{eq:prob31}\\
   c_2v_1^*+c_3{G^*}^Te+{G^*}^T\gamma&=0,\label{eq:prob32}\\
    Y-Gv_1+\epsilon e+G^*v_1^*&=0 \label{eq:prob33}
\end{align} for the optimization problem (\ref{eq:primal_nonlpropo1}). Likewise the following equations are obtained for the optimization problem (\ref{eq:primal_nonlpropo2}):
\begin{align}
    c_4v_2+G^T(Gv_2-Y)+G^T\lambda&=0, \label{eq:prob41}\\
    c_5v_2^*+c_6{G^*}^Te+{G^*}^T\lambda&=0,\label{eq:prob42}\\
    Gv_2-Y+\epsilon_2e+G^*v_2^*&=0, \label{eq:prob43}
\end{align}
$\lambda$ being the Lagrangian multiplier. Using the aforementioned equations, (\ref{eq:prob31})-(\ref{eq:prob33}), $\gamma$ is given by the equations (\ref{eq:determine_alpha}) and by utilizing equations (\ref{eq:prob41})-(\ref{eq:prob43}) $\lambda$ is given by (\ref{eq:determine_beta}) with $G=
\begin{bmatrix}
    K(M,M^T) & e
\end{bmatrix}$ and $G^*=\begin{bmatrix}
    K(M^*,{M^*}^T) & e
\end{bmatrix}$. Other notations have the same meaning as defined for the linear case. \par
Using $\gamma$ and $\lambda$, the parameters of the hyperplane $\tilde{r}_1(x)$ and $\tilde{r}_2(x)$ are obtained. 
The target value of the unknown samples is given by the following function:
\begin{align} \label{eq:dectsvrnl}
    r(x)=\frac{1}{2}(r_1(x)+r_2(x))=\frac{1}{2}(u_1+u_2)^TK(M,x)+\frac{1}{2}(b_1+b_2).
\end{align}
The optimization problems for the linear and non-linear cases of the proposed LSTSVR-PI lead to the solution of system of linear equations, which has comparatively improved efficiency as compared to solving QPPs.  
\begin{thm}
    For any given positive $c_2,~c_3$ the solution of the problem (\ref{eq:primalpropo1}) is always unique.
\end{thm}
\begin{proof}
    Following \cite{burges1999uniqueness}, we have $F$ as the objective function:
    \begin{align*}
      F=  \frac{1}{2}c_1 (\norm{u_1}^2+b_1^2)+ \frac{1}{2}c_2(\norm{u_1^*}^2+{b_1^*}^2) + \frac{1}{2} \norm{Y-(Mu_1+eb_1)}^2+c_3e^T(M^*u_1^*+eb_1^*),
      \end{align*}
      $\text{define }w=(u_1,b_1,u_1^*,b_1^*)^T.  \text{ Assume } w_1 \text{ and } w_2$ are two solutions problem (\ref{eq:primalpropo1}). Due to the convexity of the problem, a family of the solution is given by $w_t=(1-t)w_1+tw_2,~t \in [0,1]$, and $F(w_1)=F(w_2)=F(w_t)$.\\
      Thus, $F(w_t)-F(w_1)=0$. Expanding we get,
      \begin{align*}
       &\bigg[\frac{1}{2}c_1 (\norm{u_t}^2+b_t^2)+ \frac{1}{2}c_2(\norm{u_t^*}^2+{b_t^*}^2) + \frac{1}{2} \norm{Y-(Mu_t+eb_t)}^2+c_3e^T(M^*u_t^*+eb_t^*)\bigg] \\
       -&\bigg[\frac{1}{2}c_1 (\norm{u_1}^2+b_1^2)+ \frac{1}{2}c_2(\norm{u_1^*}^2+{b_1^*}^2) + \frac{1}{2} \norm{Y-(Mu_1+eb_1)}^2+c_3e^T(M^*u_1^*+eb_1^*)\bigg]=0.
      \end{align*}
    Differentiating the above equation with respect to $t$ twice, we get
    \begin{align*}
        c_1\norm{(u_1-u_2)}^2+c_1(b_1-b_2)^2+c_2\norm{u_1^*-u_2^*}+c_2(b_1^*-b_2^*)+\norm{(Mu_1-Mu_2)(b_1e-b_2e)}^2=0.
    \end{align*}
    Hence, from the above equation, we get $u_1=u_2$, $b_1=b_2$, $u_1^*=u_2^*$, $b_1^*=b_2^*$.
\end{proof}
The same follows for the remaining proposed optimization problems (\ref{eq:primalpropo2}), (\ref{eq:primal_nonlpropo1}) and (\ref{eq:primal_nonlpropo2}). 

\subsection{Theoretical properties}
In this section, we discuss the theoretical properties of the proposed LSTSVR-PI. We provide a generalization error bound for the proposed LSTSVR-PI. The bound on generalization error relies on the Rademacher complexity, which is defined as follows:

\begin{definition} \cite{bartlett2002rademacher} For a sample set $S=\{x_1,x_2,\ldots,x_m\}$ generated by distribution $\mathscr{D}$ on set $X$ and $\mathscr{F}$ represents a family of functions which maps from $X$ to output space. Then the Rademacher complexity of $\mathscr{F}$ is defined as:
\begin{align*}
    \mathscr{R}_m(\mathscr{F})= \mathbb{E} \left[ \frac{1}{m} \underset{f \in \mathscr{F}}{\text{sup}} \sum_{i=1}^{m} \sigma_i f(x_i)\right],
\end{align*}
where $\sigma_i$ corresponds to the independent random variable chosen from $\{-1,1\}$ with equal probability, and are known as Rademacher variables.  $\mathbb{E}$ corresponds to the expectation.
\end{definition}

\begin{thm}\label{thm:1} {Rademacher complexity} 
    Let $K: X \times X \rightarrow \mathbb{R}$ be a kernel function  and $S=\{x_1,x_2,\ldots, x_m\}$ be a sample set such that $\underset{x \in X}{\text{sup }}K(x,x) < \infty$ 
    Assume $\phi: X \rightarrow H$ be the feature map for the kernel function K. Let $\mathscr{F}$ be the class of functions in kernel space such that 
    \begin{align}
        \mathscr{F_B}=\{f|f: x \rightarrow \frac{1}{2}(u_1^Tx+u_2^Tx): \norm{u_1} \leq B, \norm{u_2} \leq B\}, B \in \mathbb{R}. 
    \end{align}
    Then, the Rademacher complexity is given by \begin{align}
        R_m(\mathscr{F_B})=\frac{B}{m}\sqrt{\sum_{i=1}^mK(x_i,x_i)}.
    \end{align}
\end{thm}
\begin{proof}
    Using the definition of Rademacher complexity, we have \\
    \begin{align*}
        \mathscr{R}_m(\mathscr{F_B})&=\frac{1}{m}\mathbb{E}\left[ \underset{f \in \mathscr{F}_B}{\text{sup}} \sum_{i=1}^{m} \sigma _if(x_i) \right]\\
        &=\frac{1}{m} \mathbb{E} \left[\underset{\norm{u_1} \leq B, \norm{u_2} \leq B }{\text{sup}} \sum_{i=1}^{m} \frac{\sigma_i}{2}(u_1^T\phi(x_i)+u_2^T \phi(x_i)) \right]\\
        & \leq \frac{1}{m} \mathbb{E}\left[\underset{\norm{u_1} \leq B}{\text{sup}} \sum_{i=1}^{m} \frac{\sigma_i}{2}u_1^T\phi(x_i)        +\underset{ \norm{u_2} \leq B }{\text{sup}}\sum_{i=1}^{m} \frac{\sigma_i}{2}u_2^T\phi(x_i)\right]   \\
         & = \frac{B}{2m} \mathbb{E}\left[ \norm{\sum_{i=1}^{m} \sigma_i \phi(x_i)}_{\infty}        +\norm{\sum_{i=1}^{m} \sigma_i \phi(x_i)}_{\infty} \right]~~~~~~~~\text{(by definition of dual norm)}\\
         &=\frac{B}{m} \mathbb{E}\left[ \norm{\sum_{i=1}^{m} \sigma_i \phi(x_i)} \right]~~~~~~~~~~~~~~~~~~~~~~~~~~~~~~~~~~\text{(using the norm equivalence)}\\
         &=\frac{B}{m} \mathbb{E} \left[ \sqrt{\langle \sum_{i=1}^{m}\sigma_i \phi(x_i), \sum_{i=1}^{m}\sigma_i \phi(x_i) \rangle}\right] ~~~~~~~~~~~~\text{(using Cauchy Schwarz inequality)}\\
         &=\frac{B}{m} \mathbb{E}\left[ \sqrt{\sum_{1 \leq i,j \leq m} \sigma_i, \sigma_j \langle \phi(x_i), \phi(x_j)\rangle} \right]\\
         & = \frac{B}{m} \left[ \sqrt{\mathbb{E} \left[\sum_{1 \leq i\leq m} \langle \phi(x_i), \phi(x_j)\rangle\right]} \right]~~~~~~~~~~~~~~~~~~~~~~~\text{(using Jensen's inequality)}\\
         & = \frac{B}{m} \left[ \sqrt{\sum_{1 \leq i\leq m} K(x_i,x_i)} \right]
    \end{align*}   
\end{proof}
\begin{lemma}{\label{lemma:1}}
    The loss function $l$ in LSTSVR-PI exhibits Lipschitz continuity, i.e., there exists a positive Lipschitz constant $L$, such that, $\forall x, y \in \mathbb{R}^n$
    \begin{align}
        \norm{l(x)-l(y)} \leq L \norm{x-y}.
    \end{align}
\end{lemma}
\begin {proof}
For the two hyperplanes, we have defined the loss function in two different ways: $l_1(x)=x-\epsilon_1 e$ and  $l_2(x)=x-\epsilon_2 e$. It is straightforward that both $l_1$ and $l_2$ are Lipschitz continuous with Lipschitz constant $L_1$ and $L_2$. respectively. Thus, the loss function $l$ of LSTSVR-PI is Lipschitz continuous.   
\end{proof}
The general formula for the generalization error bound based on the Rademacher complexity can be expressed using the following theorem.
\begin{thm}{\label{thm:2}}
\cite{bartlett2002rademacher} Consider $\mathscr{G}(f)$ and $\mathscr{E}(f)$ as the generalization error bound and the empirical error bound, respectively, and $\mathscr{F}$ as a family of functions, with a probability of at least $1-\delta ~(\delta \in (0,1))$ over the $m $ samples and for all $f \in \mathscr{F}$, the Lipschitz continuous loss function 
$l$ with Lipschitz constant $L$, the following relation holds.
\begin{align}
    \mathscr{G}(f ) \leq \mathscr{E}(f ) + 2L \mathscr{R}_m (\mathscr{F}) + \sqrt{\frac{\text{ln}(1/\delta)}{2m}}
\end{align}
\end{thm}
Using the above theorems, the generalization error of LSTSVR-PI is bounded which is given by the following theorem. 
\begin{thm}
Let $\mathscr{\tilde{H}}$ be a family of functions such that $\tilde{\mathscr{H}}=\{\tilde{h}|\tilde{h}:x \rightarrow f(x)-g(x)+(\epsilon_2 e -\epsilon_1 e) \}$, where $\epsilon_1$, $\epsilon_2 \in \mathbb{R}$, $f \in \mathscr{F}_B$ and $g(x)$ corresponds to the true label of $x$.  For Lipschitz continuous loss functions $l$ of proposed LSTSVR-PI with Lipschitz constant $L$, with probability at least $1-\delta$ $(\delta \in (0, 1))$ over the $m$ samples. The generalization error for all $\tilde{h} \in \mathscr{\tilde{H}}$ is bounded by 
\begin{align}
    \mathscr{G}(\tilde{h}) \leq \mathscr{E}(\tilde{h}) +\frac{2LB}{m} \sqrt{\sum_{i=1}^{m}K(x_i,x_i)} + \sqrt{\frac{\text{ln}(1/\delta)}{2m}},
\end{align} where the definition of kernel function follows from Theorem 2. 
\end{thm}
\begin{proof} By definition of Rademacher complexity, we have
\begin{align*}
    \mathscr{R}_m(\tilde{\mathscr{H}})&=\frac{1}{m} \mathbb{E} \left[ \underset{\tilde{h} \in \tilde{\mathscr{H}}}{\text{sup}}\sum_{i=1}^{m} \sigma_i \tilde{h}_i\right]\\
    &=\frac{1}{m} \mathbb{E} \left[ \underset{\tilde{h} \in \tilde{\mathscr{H}}}{\text{sup}}\sum_{i=1}^{m} (\sigma_i f(x_i)- \sigma_i g(x_i)+\sigma_i (\epsilon_2 e -\epsilon_1 e)) \right]\\
    &=\frac{1}{m} \mathbb{E} \left[ \underset{f \in {\mathscr{F}_B}}{\text{sup}}\sum_{i=1}^{m} (\sigma_i f(x_i)- \sigma_i g(x_i)+ \sigma_i (\epsilon_2 e -\epsilon_1 e)) \right]\\
    &=\frac{1}{m} \mathbb{E} \left[ \underset{f \in {\mathscr{F}_B}}{\text{sup}}\sum_{i=1}^{m} \sigma_i f(x_i)\right]  -E \left[\sum_{i=1}^{m} \sigma_i (g(x_i)+(\epsilon_2 e -\epsilon_1 e)) \right]\\
    &=\frac{1}{m} \mathbb{E} \left[ \underset{f \in {\mathscr{F}}_B}{\text{sup}}\sum_{i=1}^{m} \sigma_i f(x_i)\right]\\
    &=\mathscr{R}_m (\mathscr{F}_B)
\end{align*}
    From Theorem 3, the relation between the generalization error bound and the empirical error found is obtained. Besides using theorem 2, substitute the upper bound of Theorem Rademacher complexity. Then the desired result is obtained.
\end{proof}

\section{Numerical Experiments}{\label{sec:exp}}
In this section, we focus our attention on the numerical experiments performed to compare the performance of the proposed LSTSVR-PI with the baseline models such as TSVR \cite{peng2010tsvr}, TWSVR \cite{khemchandani2016twsvr}, LSTSVR \cite{huang2016sparse}, TSVQR \cite{ye2023twin}, and SVR+ \cite{vapnik2009new}. \\
\textbf{Experimental Setup:} All the experiments are carried out in MATLAB 2022b on a PC with 11th Gen Intel(R) Core(TM) i7-11700 @ 2.50GHz processor and 16 GB RAM. We carried out the selection of parameters using $5$ fold cross-validation with grid search. To reduce the computational cost of tuning of parameters, we considered $c_1=c_4$, $c_2=c_5$ and $c_3=c_6$ for the proposed LSTSVR-PI.  All the datasets are normalized using min-max normalization which is given as:
\begin{align}{\label{eq:exp}}
    x_{ij \text{ normalized}}=\frac{x_{ij}-x_{j\text{ min}}}{x_{j \text{ max}}-x_{j\text{ min}}},
\end{align}
where $x_{ij}$ denote $(i,j)^{th}$ element of training data matrix M; $x_{j \text{ max}}$ and $x_{j\text{ min}}$ signify the maximum and minimum element of $j^{th}$ column of matrix M. 
For the non-linear variant of our proposed model, we incorporated the Gaussian (RBF) kernel function,  ($\text{exp}(\norm{x_i-x_j}^2)/2 \mu^2$, $\mu$ denotes the kernel parameter). 
We evaluate and compare the performance of the proposed LSTSVR-PI with several baseline models including TSVR, TWSVR, LSTSVR, TSVQR, and SVR+ to assess its effectiveness and superiority across various regression datasets. To obtain privileged information for the datasets, we divided the number of features (\# features) in the dataset into two parts. We designated $\lceil\frac{\text{\# features}}{2}\rceil$ features as the normal features, while the remaining are treated as the privileged features for all the datasets.  
We carried out the experiments on artificially generated datasets and real-world UCI \cite{dua2017uci},  and stock datasets. 
The metrics considered to compare the performance of the models are 
\begin{align*}
 \text{RMSE}=&\sqrt{\sum_{j}\frac{(y_j-\tilde{y_j}^2)}{m}},\\
 \text{SSE}=&\sum_j{(y_j-{\tilde{y}_j})}^2,\\
   \frac{\text{SSE}}{\text{SST}}=&\frac{\sum_j{(y_j-{\tilde{y}_j})^2}}{\sum_j{(y_j-\bar{y})^2}}
\end{align*}
and time required (in seconds) for training, where $y_j$ is the actual target and $\tilde{y}_j$ is the predicted target. $\bar{y}=\frac{1}{m}\sum_j y_j$ corresponds to the mean of actual target values. SSE, and SST denote the sum of squared error and sum of square total, respectively; RMSE denotes the root mean squared error.

\subsection{Parameter selection}
The different parameters involve in the proposed LSTSVR-PI are regularization hyperparameters $c_1=c_4$, $c_2=c_5$, $c_3=c_6$, $\epsilon$ and the kernel parameter $\sigma$.  Following \cite{shao2013varepsilon}, all the parameters (except $\epsilon$) are tuned in the range $\{2^i|i=-8,-7,\ldots,7,8\}$ for the synthetic, UCI and stock datasets. The respective regularization parameters $c_1=c_3$, $c_2=c_4$,  and $\sigma$ for the baseline models, i.e.,  TSVR, TWSVR, LSTSVR, TSVQR and SVR+ are selected from the range $\{2^i|i=-8,-7,\ldots,7,8\}$. The quantile parameter $\tau$ of TSVQR is chosen from the range $[0.1,0.25,0.5,0.75,0.9]$ following the article \cite{ye2023twin}. The value of $\epsilon_1$ and $\epsilon_2$ are fixed at $0.01$ \cite{tanveer2017regularization}. 

\subsection{Artificially generated datasets}
To facilitate a comprehensive comparison between the baseline models and the proposed LSTSVR-PI, we conducted experiments using four artificially generated datasets. The functions employed for these datasets are listed in Table \ref{tab:different function}. In our experimental setup, we generated $100$ random training samples and $200$ testing samples for each type of function. To facilitate a robust comparison, the training samples are intentionally contaminated with three distinct types of noise levels \cite{tanveer2016efficient}, as follows: 
\begin{enumerate}
    \item Uniform noise over the interval $[-0.2,0.2]$;
    \item Gaussian noise with mean $0$ and standard deviation $0.05$;
    \item Gaussian noise with mean $0$ and standard deviation $0.2$;
\end{enumerate}
This approach allows us to assess the performance of the models under different noise conditions for a comprehensive evaluation.


\begin{table}[h!]
\begin{center}
\caption{Different functions to generate artificial datasets.}
\label{tab:different function}
\resizebox{15cm}{!}{
\begin{tabular}{ l c c } 
 \hline
 Name & Function & Domain of the function \\
\hline
Function 1 & $f(x_1,x_2)=\Large{\frac{\text{sin}(\sqrt{x_1^2+x_2^2})}{\sqrt{x_1^2+x_2^2}}}$  & $x\in [-4 \pi, 4 \pi]$ \\
 \hline
 Function 2 & $f(x_1,x_2,x_3,x_4,x_5)=10 \text{ sin}(\pi x_1x_2)+20(x_3-0.5)^2+10x_4+5x_5$ & $x_1,x_2,x_3,x_4,x_5 \in [0,1]$ \\
 \hline
 Function 3 &$f(x_1,x_2)=\text{exp}(x_1 \text{sin}(\pi x_2))$ & $x_1,x_2 \in [-1,1]$ \\
 \hline
 Function 4 & $f(x_1,x_2)=1.9[1.35+\text{exp}(x_1) \text{sin}(13(x_1-0.6)^2)+\text{exp}(3(x_2-0.5)) \text{sin}(4 \pi(x_2-0.9)^2)]$  & $x_1,x_2 \in [0,1]$\\
 \hline
\end{tabular}}
\end{center}
\end{table}

The training samples are contaminated with the different generated noise and the test samples are noise-free.  Four independent experiments are performed on each dataset with different types of noise and the average values are reported in Table \ref{tab:synthetic dataset results}. The least RMSE, SSE, SSE/SST value correspond to the best performance. 
The average RMSE values for artificially generated datasets are $0.1433$, $0.1406$, $0.1433$, $0.1405$, $0.1499$, and $0.135$ for TSVR, TWSVR, LSTSVR, TSVQR, SVR+ and proposed LSTSVR-PI, respectively.  It's important to note that the average RMSE value may not always provide an optimal assessment because a higher RMSE value for one dataset may be offset by a lower value for another. To address this, each dataset is ranked according to the performance across different models and metrics such that the best model gets the least rank and vice versa. The average ranks corresponding to RMSE values in the aforementioned order of models are $4.83$, $2.58$, $4.75$, $2.67$, $5$, and $1.17$, respectively. It reflects that the proposed LSTSVR-PI performs well on the artificially generated datasets. Associated with the optimal parameters of RMSE across all the datasets, we calculated the SSE and SSE/SST metrics. The average SSE and SSE/SST of the proposed LSTSVR-PI are $0.4$ and $0.6833$, respectively, which are the least values among all the models. The average rank of the proposed LSTSVR-PI corresponding to SSE is 1.08, which is the best.

\begin{table*}[h!]
 \centering
 \caption{\vspace{0.1mm}RMSE values along with SSE, SSE/SST and training time across the artificially generated synthetic datasets with different types of noise for TSVR, TWSVR, LSTSVR, TSVQR, SVR+ and the proposed LSTSVR-PI.} 
 \resizebox{13cm}{!}{
\label{tab:synthetic dataset results}
\begin{tabular}{|lccccccc|}
\hline
Dataset &Metric &TSVR \cite{peng2010tsvr} &	TWSVR \cite{khemchandani2016twsvr} &	LSTSVR \cite{huang2016sparse} & TSVQR \cite{ye2023twin} & SVR+ \cite{vapnik2009new} & Proposed	LSTSVR-PI \\
\hline
\multicolumn{8}{|c|}{Type 1: Uniform noise over the interval $[-0.2,0.2]$.}\\
\hline
function 1 &RMSE&$0.1341$&$0.1329$&$0.1347$&$0.131$&$0.144$&$0.1264$\\
&SSE&$0.38$&$0.3737$&$0.3825$&$0.3626$&$0.4452$&$0.3305$\\
&SSE/SST&$0.9925$&$0.9914$&$0.9927$&$0.9891$&$0.9887$&$0.9998$\\
&Time&$0.0119$&$0.007$&$0.0048$&$0.0088$&$0.0468$&$0.0082$\\
function 2 &RMSE&$0.1431$&$0.1397$&$0.1419$&$0.1428$&$0.143$&$0.1298$\\
&SSE&$0.4147$&$0.3974$&$0.4106$&$0.4154$&$0.4178$&$0.3405$\\
&SSE/SST&$0.5898$&$0.5685$&$0.5855$&$0.5926$&$0.6033$&$0.4944$\\
&Time&$0.0105$&$0.008$&$0.0034$&$0.0182$&$0.0352$&$0.0044$\\
function 3 &RMSE&$0.1974$&$0.1951$&$0.199$&$0.1908$&$0.1947$&$0.1928$\\
&SSE&$0.8003$&$0.7876$&$0.8117$&$0.7573$&$0.7805$&$0.7644$\\
&SSE/SST&$0.9727$&$0.9587$&$0.9897$&$0.9463$&$0.9818$&$0.9845$\\
&Time&$0.0133$&$0.0043$&$0.002$&$0.0023$&$0.0621$&$0.0042$\\
function 4 &RMSE&$0.1045$&$0.1025$&$0.1053$&$0.104$&$0.1301$&$0.0979$\\
&SSE&$0.2212$&$0.2124$&$0.2247$&$0.2196$&$0.3454$&$0.1945$\\
&SSE/SST&$0.2827$&$0.2709$&$0.2871$&$0.2809$&$0.4323$&$0.2505$\\
&Time&$0.0155$&$0.0113$&$0.0048$&$0.0127$&$0.0385$&$0.0037$\\
\hline
\hline
\multicolumn{8}{|c|}{Type 2: Gaussian noise with mean $0$ and standard deviation $0.05$.}\\
\hline
function 1 &RMSE&$0.1047$&$0.1023$&$0.1052$&$0.1022$&$0.1034$&$0.0963$\\
&SSE&$0.2457$&$0.2296$&$0.2478$&$0.228$&$0.2377$&$0.2025$\\
&SSE/SST&$0.9826$&$0.9269$&$0.9892$&$0.9326$&$0.9637$&$0.9889$\\
&Time&$0.0148$&$0.0036$&$0.0042$&$0.0081$&$0.0427$&$0.005$\\
function 2 &RMSE&$0.1433$&$0.1401$&$0.1421$&$0.1424$&$0.1439$&$0.1301$\\
&SSE&$0.4157$&$0.399$&$0.4111$&$0.4131$&$0.4226$&$0.3418$\\
&SSE/SST&$0.5863$&$0.5655$&$0.5813$&$0.5869$&$0.6058$&$0.4934$\\
&Time&$0.011$&$0.0077$&$0.0022$&$0.0153$&$0.0369$&$0.0045$\\
function 3 &RMSE&$0.1997$&$0.197$&$0.2009$&$0.1934$&$0.1986$&$0.1918$\\
&SSE&$0.8224$&$0.8062$&$0.8292$&$0.7823$&$0.8141$&$0.7608$\\
&SSE/SST&$0.9848$&$0.9726$&$0.9948$&$0.9688$&$0.9821$&$0.9754$\\
&Time&$0.0125$&$0.0034$&$0.0015$&$0.014$&$0.072$&$0.0042$\\
function 4 &RMSE&$0.1042$&$0.1008$&$0.1043$&$0.1027$&$0.1301$&$0.0973$\\
&SSE&$0.2193$&$0.2055$&$0.2202$&$0.2152$&$0.345$&$0.1921$\\
&SSE/SST&$0.2843$&$0.2671$&$0.2858$&$0.2798$&$0.4313$&$0.2509$\\
&Time&$0.0173$&$0.0065$&$0.0017$&$0.0097$&$0.0397$&$0.004$\\
\hline
\hline
\multicolumn{8}{|c|}{Type 3: Gaussian noise with mean $0$ and standard deviation $0.2$.}\\
\hline
function 1 &RMSE&$0.1452$&$0.1432$&$0.1445$&$0.1422$&$0.1443$&$0.1411$\\
&SSE&$0.4363$&$0.4243$&$0.4332$&$0.4199$&$0.4332$&$0.4125$\\
&SSE/SST&$1.01$&$0.9811$&$1.0024$&$0.9895$&$0.9902$&$1.0102$\\
&Time&$0.0112$&$0.005$&$0.0058$&$0.0078$&$0.0723$&$0.0045$\\
function 2 &RMSE&$0.1425$&$0.1397$&$0.1414$&$0.1415$&$0.1432$&$0.1299$\\
&SSE&$0.4108$&$0.3969$&$0.4072$&$0.4084$&$0.4191$&$0.341$\\
&SSE/SST&$0.5828$&$0.5654$&$0.5794$&$0.5854$&$0.6031$&$0.4948$\\
&Time&$0.011$&$0.005$&$0.0045$&$0.0126$&$0.032$&$0.0041$\\
function 3 &RMSE&$0.1957$&$0.1914$&$0.1948$&$0.1875$&$0.1921$&$0.1878$\\
&SSE&$0.7864$&$0.7549$&$0.7762$&$0.7207$&$0.7593$&$0.7185$\\
&SSE/SST&$1.0072$&$0.9675$&$0.9981$&$0.9609$&$0.9904$&$0.9901$\\
&Time&$0.0119$&$0.0049$&$0.0039$&$0.0055$&$0.0708$&$0.0034$\\
function 4 &RMSE&$0.105$&$0.1019$&$0.1053$&$0.1051$&$0.1319$&$0.0993$\\
&SSE&$0.2233$&$0.2104$&$0.2253$&$0.2247$&$0.3544$&$0.2004$\\
&SSE/SST&$0.2949$&$0.2777$&$0.2978$&$0.2958$&$0.4616$&$0.2662$\\
&Time&$0.0152$&$0.0056$&$0.0075$&$0.0117$&$0.0393$&$0.0043$\\
\hline
\hline
Average  &RMSE& 0.1433&	0.1406&	0.1433&	0.1405&	0.1499	&0.135\\
  &SSE& 0.448&	0.4332&	0.4483&	0.4306&	0.4812&	0.4\\
  &SSE/SST& 0.7142 &	0.6928&	0.7153&	0.7007&	0.7529	&0.6833\\
  \hline
Average rank & RMSE & 4.83 &	2.58&	4.75&	2.67&	5&	1.17\\
 & SSE & 4.67&	2.58&	4.71&	2.83&	5.13&	1.08\\
\hline
\end{tabular}}
\end{table*}

\subsection{Real-world datasets}
In this subsection, we present the results of numerical experiments involving the proposed LSTSVR-PI model and baseline models across 21 real-world UCI, KEEL and stock datasets. 
Following UCI and KEEL datasets are considered: Auto-price, Daily\_Demand\_Forecasting\_Orders, auto-original, bodyfat,  hungary chickenpox, Istanbul stock exchange data, qsar aquatic toxicity, servo, slump\_test, wpbc, yacht hydrodynamics, autoMPG, baseball, ele-1, dee, machineCPU which are split in a $70:30$ ratio for training and testing, respectively. 

The stock datasets are IBM, NVDA, citigroup,  Infosys, microsoft and wipro, which are publicly available at  \href{https://finance.yahoo.com/}{`Link'}. The aforementioned stock index financial datasets are derived from the low prices of stocks and are designed to predict the current value using the five previous values. These datasets span from February 2, 2018, to December 30, 2020, encompassing a total of 738 or 755 samples. In our experimental setup, the first 200 samples are designated as the training data, while the remaining are reserved for testing the performance of the model. 

Table \ref{tab:UCI_Stock_KEEL dataset results} depicts the performance of the proposed LSTSVR-PI and existing models TSVR, TWSVR, LSTSVR, TSVQR, and SVR+ in terms of metric RMSE, SSE, SSE/SST and training time (in seconds). Table \ref{tab:UCI_Stock_KEEL dataset results} 
 also includes the number of samples and features for ll the datasets. The performance evaluation based on the RMSE indicates that lower RMSE values correspond to better model performance.  The proposed LSTSVR-PI model achieves the least average RMSE value of $0.0879$, followed by LSTSVR, TSVR, TSVQR, TWSVR, and SVR+ with an average RMSE of $0.0881$, $0.0883$, $0.0931$, $0.1095$, and $0.1171$, respectively. These values indicate that the proposed model exhibits the best performance among all the models.  
 The average rank corresponding to the RMSE value for the proposed LSTSVR-PI is $2.14$, which is the least. Further, we calculated the average of SSE and SSE/SST metrics obtained across all the datasets corresponding to optimal parameters of RMSE. The average SSE value for the proposed LSTSVR-PI is $0.5267$, the second-best average SSE value is $0.5295$ for LSTSVR, then $0.5311$ for TSVR. The proposed LSTSVR-PI has an average rank of $2.29$, analogous to SSE which is the least among the models. The average rank corresponding to SSE/SST for the proposed LSTSVR-PI is $2.26$, then $2.55$ for TSVR. The average rank of the remaining models in increasing order are as follows: 3.1 for LSTSVR, 3.4 for TSVQR, 3.93 for TWSVR, and 5.76 for SVR+. Thus, the average ranks average RMSE, and SSE values depict the proposed LSTSVR-PI has superior performance.  

\begin{landscape}
\begin{footnotesize}
\begin{longtable}[t]{lccccccc}
\caption{RMSE values along with SSE, SSE/SST and training time across the real-world datasets for TSVR, TWSVR, LSTSVR, TSVQR, SVR+ and the proposed LSTSVR-PI.}
    \label{tab:UCI_Stock_KEEL dataset results}\\
\hline
Dataset &Metric &TSVR \cite{peng2010tsvr} &	TWSVR \cite{khemchandani2016twsvr} &	LSTSVR \cite{huang2016sparse} & TSVQR \cite{ye2023twin} & SVR+ \cite{vapnik2009new} & Proposed	LSTSVR-PI \\
(\#samples, \#features) &&&&&&&\\
\hline 
\endfirsthead
\multicolumn{8}{c}%
{{\bfseries \tablename\ \thetable{} -- continued from previous page}} \\
\hline
Dataset &Metric &TSVR \cite{peng2010tsvr} &	TWSVR \cite{khemchandani2016twsvr} &	LSTSVR \cite{huang2016sparse} & TSVQR \cite{ye2023twin} & SVR+ \cite{vapnik2009new} & Proposed	LSTSVR-PI \\ 
(\#samples, \#features) &&&&&&&\\
 \hline 
\endhead
\hline \multicolumn{8}{|r|}{{Continued on next page}} \\ \hline
\endfoot
\endlastfoot
Auto-price&RMSE&$0.084$&$0.0892$&$0.0858$&$0.0916$&$0.1661$&$0.0855$\\
(159,8)&SSE&$0.1596$&$0.1827$&$0.1663$&$0.1895$&$0.6381$&$0.1692$\\
&SSE/SST&$0.1695$&$0.1938$&$0.1779$&$0.1953$&$0.6678$&$0.1788$\\
&Time&$0.0075$&$0.0092$&$0.0032$&$0.009$&$0.0201$&$0.0039$\\
Daily\_Demand\_Forecasting\_Orders&RMSE&$0.0386$&$0.0444$&$0.0339$&$0.0388$&$0.1182$&$0.0359$\\
(60,6)&SSE&$0.0131$&$0.0177$&$0.0104$&$0.0129$&$0.134$&$0.011$\\
&SSE/SST&$0.0448$&$0.0606$&$0.0357$&$0.0444$&$0.4031$&$0.0378$\\
&Time&$0.0038$&$0.0018$&$0.0011$&$0.0095$&$0.0066$&$0.0022$\\
auto-original&RMSE&$0.1035$&$0.1272$&$0.104$&$0.1049$&$0.1548$&$0.1037$\\
(392,4)&SSE&$0.5891$&$0.89$&$0.5953$&$0.6075$&$1.3182$&$0.5929$\\
&SSE/SST&$0.2602$&$0.3689$&$0.2636$&$0.2676$&$0.5727$&$0.262$\\
&Time&$0.0281$&$0.0352$&$0.0138$&$0.0193$&$0.2344$&$0.0099$\\
bodyfat&RMSE&$0.0249$&$0.031$&$0.0266$&$0.026$&$0.1016$&$0.0265$\\
(252,7)&SSE&$0.0361$&$0.0445$&$0.0377$&$0.0369$&$0.3817$&$0.0372$\\
&SSE/SST&$0.038$&$0.0468$&$0.0397$&$0.0389$&$0.3938$&$0.0393$\\
&Time&$0.013$&$0.0164$&$0.0051$&$0.0133$&$0.0585$&$0.0066$\\
hungary chickenpox&RMSE&$0.0838$&$0.1047$&$0.0841$&$0.0823$&$0.0976$&$0.0848$\\
(522,10)&SSE&$0.5652$&$0.8193$&$0.5675$&$0.5419$&$0.7435$&$0.5787$\\
&SSE/SST&$0.72$&$0.8981$&$0.7338$&$0.6915$&$0.9422$&$0.7358$\\
&Time&$0.0495$&$0.0596$&$0.0172$&$0.0267$&$0.3751$&$0.0167$\\
istanbul stock exchange data&RMSE&$0.0747$&$0.0807$&$0.0752$&$0.0884$&$0.0967$&$0.0751$\\
(536,4)&SSE&$0.4211$&$0.4898$&$0.4272$&$0.6032$&$0.7093$&$0.4262$\\
&SSE/SST&$0.4139$&$0.4773$&$0.4193$&$0.498$&$0.6167$&$0.4183$\\
&Time&$0.0424$&$0.0373$&$0.0264$&$0.0256$&$0.357$&$0.0178$\\
qsar\_aquatic toxicity&RMSE&$0.1178$&$0.1177$&$0.1182$&$0.1205$&$0.1349$&$0.118$\\
(546,4)&SSE&$1.0672$&$1.069$&$1.0739$&$1.1196$&$1.4029$&$1.0685$\\
&SSE/SST&$0.5142$&$0.5119$&$0.5176$&$0.5292$&$0.6795$&$0.5146$\\
&Time&$0.0357$&$0.0411$&$0.0218$&$0.0201$&$0.4544$&$0.0166$\\
servo&RMSE&$0.2306$&$0.3206$&$0.2307$&$0.2604$&$0.2294$&$0.2282$\\
(167,2)&SSE&$1.288$&$2.4065$&$1.2823$&$1.6016$&$1.2539$&$1.2562$\\
&SSE/SST&$0.9578$&$1.0269$&$0.9703$&$0.9564$&$0.9587$&$0.9454$\\
&Time&$0.0057$&$0.0086$&$0.0025$&$0.0009$&$0.0217$&$0.003$\\
slump\_test&RMSE&$0.0499$&$0.053$&$0.0493$&$0.0502$&$0.1098$&$0.05$\\
(103,5)&SSE&$0.038$&$0.043$&$0.0386$&$0.04$&$0.1774$&$0.0392$\\
&SSE/SST&$0.0823$&$0.0932$&$0.0859$&$0.0868$&$0.3598$&$0.0845$\\
&Time&$0.0071$&$0.0013$&$0.0015$&$0.0111$&$0.015$&$0.0013$\\
wpbc&RMSE&$0.2027$&$0.383$&$0.2028$&$0.2463$&$0.2015$&$0.2001$\\
(194,17)&SSE&$1.1782$&$3.9919$&$1.1768$&$1.7166$&$1.1626$&$1.1495$\\
&SSE/SST&$1.0004$&$0.9998$&$1.0062$&$0.9964$&$1$&$0.9998$\\
&Time&$0.0092$&$0.0171$&$0.0051$&$0.0095$&$0.1974$&$0.0049$\\
yacht\_hydrodynamics&RMSE&$0.2522$&$0.2521$&$0.2522$&$0.2525$&$0.2564$&$0.2519$\\
(308,3)&SSE&$2.7814$&$2.779$&$2.7814$&$2.7852$&$2.881$&$2.7723$\\
&SSE/SST&$1$&$1$&$1.0006$&$1$&$1.0037$&$0.9977$\\
&Time&$0.017$&$0.0157$&$0.0078$&$0.0124$&$0.1063$&$0.0064$\\
autoMPG6&RMSE&$0.102$&$0.1159$&$0.1014$&$0.1023$&$0.154$&$0.1016$\\
(392,3)&SSE&$0.5817$&$0.7413$&$0.5687$&$0.5801$&$1.3189$&$0.5712$\\
&SSE/SST&$0.2622$&$0.3228$&$0.2566$&$0.2603$&$0.5894$&$0.2574$\\
&Time&$0.0202$&$0.0266$&$0.0111$&$0.015$&$0.1812$&$0.0133$\\
baseball&RMSE&$0.1566$&$0.1539$&$0.154$&$0.1547$&$0.1838$&$0.1531$\\
(337,8)&SSE&$1.1605$&$1.1193$&$1.121$&$1.1312$&$1.5987$&$1.108$\\
&SSE/SST&$0.5788$&$0.5592$&$0.5649$&$0.5677$&$0.7888$&$0.5541$\\
&Time&$0.0191$&$0.0158$&$0.0091$&$0.0134$&$0.1185$&$0.0068$\\
ele-1&RMSE&$0.1187$&$0.2024$&$0.1186$&$0.1195$&$0.1329$&$0.1188$\\
(495,1)&SSE&$0.9822$&$3.5617$&$0.9803$&$0.9953$&$1.231$&$0.9839$\\
&SSE/SST&$0.7295$&$1.9749$&$0.7398$&$0.7412$&$0.9115$&$0.7313$\\
&Time&$0.0312$&$0.0718$&$0.0236$&$0.0279$&$0.3444$&$0.0102$\\
machinCPU&RMSE&$0.0395$&$0.0386$&$0.0409$&$0.0413$&$0.0571$&$0.0399$\\
(209,3)&SSE&$0.0508$&$0.0488$&$0.0531$&$0.0596$&$0.1092$&$0.0575$\\
&SSE/SST&$0.2102$&$0.2025$&$0.2207$&$0.2448$&$1.0084$&$0.2399$\\
&Time&$0.0109$&$0.0026$&$0.0042$&$0.0138$&$0.043$&$0.0054$\\
IBM&RMSE&$0.0469$&$0.0519$&$0.0469$&$0.0473$&$0.0601$&$0.0469$\\
(750,3)&SSE&$0.0889$&$0.1083$&$0.0892$&$0.0908$&$0.1473$&$0.0893$\\
&SSE/SST&$0.221$&$0.2617$&$0.2222$&$0.2249$&$0.3632$&$0.2219$\\
&Time&$0.0147$&$0.0169$&$0.0086$&$0.0121$&$0.0999$&$0.0057$\\
NVDA&RMSE&$0.0197$&$0.0198$&$0.0198$&$0.0197$&$0.0283$&$0.0198$\\
(750,3)&SSE&$0.0162$&$0.0162$&$0.0163$&$0.016$&$0.0327$&$0.0163$\\
&SSE/SST&$0.3116$&$0.3129$&$0.3155$&$0.3092$&$0.5905$&$0.3136$\\
&Time&$0.0122$&$0.0145$&$0.007$&$0.0134$&$0.1222$&$0.008$\\
citigroup&RMSE&$0.0355$&$0.035$&$0.0349$&$0.0359$&$0.0477$&$0.035$\\
(750,3)&SSE&$0.0523$&$0.052$&$0.0506$&$0.0542$&$0.0921$&$0.0515$\\
&SSE/SST&$0.2476$&$0.2464$&$0.2597$&$0.2531$&$0.4125$&$0.2441$\\
&Time&$0.0252$&$0.0074$&$0.008$&$0.0117$&$0.1172$&$0.0046$\\
infosys&RMSE&$0.0205$&$0.0218$&$0.0201$&$0.0215$&$0.0392$&$0.0201$\\
(750,3)&SSE&$0.0169$&$0.0193$&$0.0163$&$0.0187$&$0.0626$&$0.0163$\\
&SSE/SST&$0.1063$&$0.1204$&$0.1024$&$0.1157$&$0.3818$&$0.1021$\\
&Time&$0.0122$&$0.0211$&$0.008$&$0.0114$&$0.1308$&$0.0054$\\
microsoft&RMSE&$0.0143$&$0.0141$&$0.0143$&$0.0143$&$0.034$&$0.014$\\
(750,3)&SSE&$0.0083$&$0.0082$&$0.0084$&$0.0083$&$0.0463$&$0.0081$\\
&SSE/SST&$0.0719$&$0.0702$&$0.0722$&$0.0717$&$0.3918$&$0.0698$\\
&Time&$0.0107$&$0.005$&$0.0062$&$0.0116$&$0.0988$&$0.0062$\\
wipro&RMSE&$0.0372$&$0.0419$&$0.0372$&$0.0371$&$0.0559$&$0.0368$\\
(750,3)&SSE&$0.0588$&$0.0742$&$0.0584$&$0.0588$&$0.1259$&$0.0579$\\
&SSE/SST&$0.1743$&$0.2172$&$0.1731$&$0.1743$&$0.3644$&$0.1715$\\
&Time&$0.0144$&$0.0217$&$0.007$&$0.0117$&$0.0911$&$0.0047$\\
\hline
Average & RMSE  & 0.0883&	0.1095&	0.0881&	0.0931&	0.1171&	0.0879\\
 & SSE & 0.5311&	0.8801&	0.5295&	0.5842&	0.7413&	0.5267\\
 & SSE/SST & 0.3864&	0.4745&	0.3874&	0.3937&	0.6381&	0.3867\\
Average rank & RMSE &2.43&	4.1&	2.81&	4&	5.52&	2.14\\
& SSE & 2.57 &	4.19	&2.62&	3.86	&5.48&	2.29\\
&SSE/SST & 2.55 &	3.93&	3.1&	3.4	&5.76&	2.26\\
\hline
\end{longtable}
\end{footnotesize}
\end{landscape}

\begin{figure}
    \subcaptionbox{qsar\_aquatic toxicity  \label{fig:qsar_aquatic toxicity}} { %
      \includegraphics[width=0.55\textwidth,height=4cm]{qsar\_aquatic\_toxicity.png}}
    \hfill 
    \subcaptionbox{servo\label{fig:servo}} { %
      \includegraphics[width=0.55\textwidth,height=4cm]{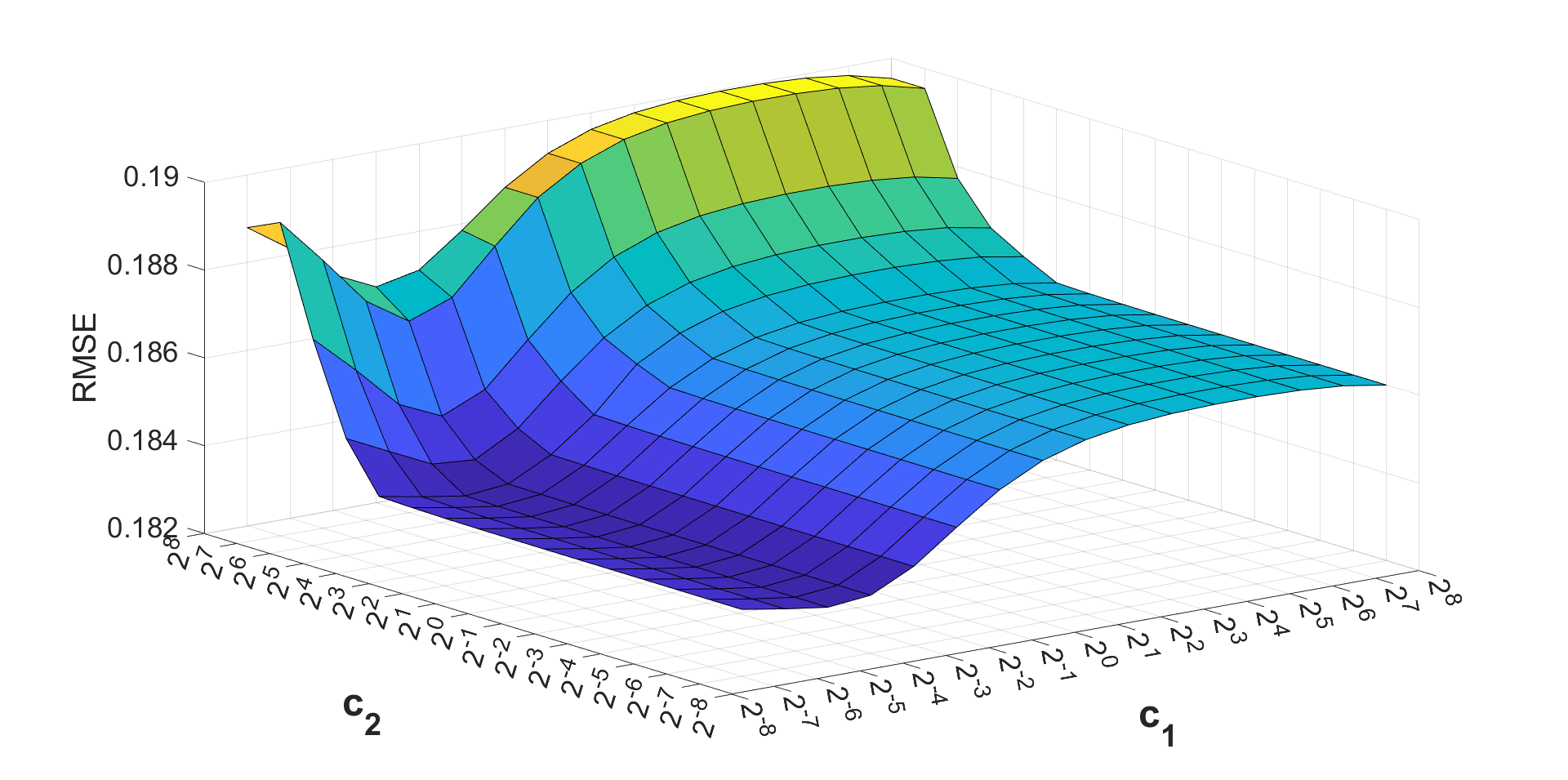}}
 
 \caption{Influence of parameters on the RMSE values of the proposed LSTSVR-PI.}
    \label{fig:sensitivity curve}
 \end{figure}

Figures \ref{fig:sensitivity curve} show the dependence of the RMSE values of the proposed LSTSVR-PI on the UCI datasets. Figure \ref{fig:qsar_aquatic toxicity} and \ref{fig:servo} denote the variation in RMSE values along with the hyperparameters $c_1$ and $c_2$. Other parameters are fixed at their optimal values. Hence, optimal RMSE values are highly dependent on the parameters.  

\subsection{Statistical analysis}{\label{subsec: statistical}}
In this subsection, we conduct a statistical assessment of the performance of the baseline models TSVR, TWSVR, LSTSVR, TSVQR, SVR+, and proposed LSTSVR-PI by employing Friedman test and Nemenyi post hoc test \cite{demvsar2006statistical}. By conducting the tests, we can thoroughly evaluate the models' performance, eliminating potential biases and enabling us to make general conclusions about their ability to generalize. The Friedman test statistically compares the models with the null hypothesis assumption that all models are equivalent and the Nemenyi test compares the models pairwise. Assuming $l$ and $n$ signify the number of models and number of datasets, respectively, the statistic is calculated using the relation:
\begin{align*}
    \chi_F^2=\frac{12n}{l(l+1)}\Bigg[ \sum _{i=1}^{l} r_i^2-\frac{l(l+1)^2}{4}\Bigg],
\end{align*}
where $r_i$ signifies the average rank of $i^{th}$ model. $\chi_F^2$ has $(l-1)$ degree of freedom. The undesirable conservative nature of Friedman statistics can be overcome by F-distribution \cite{demvsar2006statistical}.   It has degree of freedom $((l-1), (l-1)(n-1))$ and is obtained through the relation:
\begin{align*}
    F_F=\frac{(n-1)\chi_F^2}{n(l-1)-\chi_F^2}. 
\end{align*}
We analyze the models in a pairwise way by using the Nemenyi test which concludes that the models are significantly different if the difference in their ranks is greater than the critical difference \emph{(cd)}, where 
\begin{align*}
    cd=q_\alpha \sqrt{\frac{l(l+1)}{6n}}
\end{align*} 
with $q_\alpha=2.850$ at $5\%$ level of significance i.e., $\alpha=0.05$.\\
\textbf{Artificially generated datasets: } First, we statistically analyse the artificially generated datasets for which we have $l=6$ and $n=12$. On calculation using the average of ranks corresponding to RMSE, we obtain $\chi_F^2= 43.0135$ and $F_F=27.8544$ with degree of freedom $5$ and $(5,55)$, respectively. $F_{(5,55)}=2.3828$ $<$ $F_F$. Thus, the null hypothesis does not hold and the models are not equivalent. Further, we employ the Nemenyi test and calculate $cd=2.1767$.    
Corresponding to the average of ranks assigned using SSE values, we have $\chi_F^2= 43.3426$,  $F_F= 28.6221
$ and $cd=2.1767$. We observe the $F_F$ corresponding to average of the ranks assigned using SSE values is greater than $F_{(5,55)}$, thus the null hypothesis is rejected for SSE case. Table \ref{tab:Nemenyiposthoc_synthetic} represents the models which are significantly different from the proposed LSTSVR-PI using the Nemenyi test over artificially generated datasets for RMSE and SSE metrics. Table \ref{tab:Nemenyiposthoc_synthetic} infers that TWSVR and TSVQR are not significantly different than the proposed LSTSVR-PI, however, statistical tests on other datasets establish the significant difference of the model. \\
\begin{table}
\caption{Nemenyi post hoc significant difference for the proposed LSTSVR-PI with respect to the baseline models on artificially generated datasets corresponding to different metrics.}
\resizebox{1.0\textwidth}{!}{
    \begin{tabular}{|lccccc|}
    \hline 
    Metric &TSVR \cite{peng2010tsvr} & TWSVR \cite{khemchandani2016twsvr} & LSTSVR \cite{huang2016sparse} &  TSVQR \cite{ye2023twin} & SVR+ \cite{vapnik2009new}\\
    \hline
     RMSE &Yes &No& Yes &No & Yes \\
     SSE&Yes &No& Yes &No & Yes \\
         \hline
    \end{tabular}}    
    \label{tab:Nemenyiposthoc_synthetic}
\end{table}
\begin{table}
\caption{Nemenyi post hoc significant difference for the proposed LSTSVR-PI with respect to the baseline models on real-world datasets corresponding to different metrics.}
\resizebox{1.0\textwidth}{!}{
    \begin{tabular}{|lccccc|}
    \hline 
    Metric &TSVR \cite{peng2010tsvr} & TWSVR \cite{khemchandani2016twsvr} & LSTSVR \cite{huang2016sparse} &  TSVQR \cite{ye2023twin} & SVR+ \cite{vapnik2009new}\\
    \hline
         RMSE &No &Yes& No &Yes & Yes \\
     SSE &No& Yes &No & No&Yes \\
     SSE/SST &No& Yes &No & No&Yes \\
         \hline
    \end{tabular}}    
    \label{tab:Nemenyiposthoc_real_world}
\end{table}
%
\textbf{Real-world datasets: } For the numerical experiments on real-world datasets, we have $n=21$ and $l=6$. Thus degree of freedom of $\chi_F^2$ and $F_F$ are obtained as $5$ and $(5,100)$, respectively. At $5\%$ level of significance, $F_{(5,100)}=2.3053$.  Using the average ranks of the model based on the RMSE values depicted in Table \ref{tab:UCI_Stock_KEEL dataset results}, we calculate $\chi_F^2= 48.9660$, $F_F=17.4772$ and $cd=1.6454$. Corresponding to the average ranks assigned to the models based on SSE values, we have $\chi_F^2= 46.1970$, $F_F= 15.7125$, and $cd=1.6454$.  Emphasizing the average of the ranks assigned using SSE/SST values, we have $\chi_F^2= 47.4156$, $F_F=16.4682$ and $cd=1.6454$. 
For the three $F_F$ values obtained using RMSE, SSE, and SSE/SST metric, we have $F_F> F_{(5,100)}$, thus the null hypothesis of the Friedman test is rejected and models are not equivalence. Further Table \ref{tab:Nemenyiposthoc_real_world} shows which models are significantly different from the proposed LSTSVR-PI corresponding to the metric RMSE, SSE, SSE/SST. Table \ref{tab:Nemenyiposthoc_real_world} depicts TSVR and LSTSVR fails to be significantly different from the proposed LSTSVR-PI corresponding to the three metrics. However, statistical tests on different datasets prove the significant difference between the models. Thus, stronger numerical results and consistently favourable outcomes from statistical tests provide significant support for the claim that the proposed LSTSVR-PI performs better in generalization performance than the baseline models. 


\begin{table*}[h!]
 \centering
 \caption{\vspace{0.1mm}RMSE values along with SSE, SSE/SST and training time across the real-world time series datasets for TSVR, TWSVR, LSTSVR, TSVQR, SVR+ and the proposed LSTSVR-PI.} 
 \resizebox{12.5cm}{!}{
\label{tab:time_series_data}
\begin{tabular}{|lccccccc|}
\hline
Dataset &Metric &TSVR \cite{peng2010tsvr} &	TWSVR \cite{khemchandani2016twsvr} &	LSTSVR \cite{huang2016sparse} & TSVQR \cite{ye2023twin} & SVR+ \cite{vapnik2009new} & Proposed	LSTSVR-PI \\
(\#samples, \#features) & & & & & & &\\
  & & & & & & &\\
\hline
NNGC1\_dataset\_D1\_V1\_003&RMSE&$0.2272$&$0.1564$&$0.1569$&$0.2371$&$0.1657$&$0.1439$\\
(430,3)&SSE&$3.1422$&$1.4744$&$1.493$&$3.3962$&$1.6633$&$1.2493$\\
&SSE/SST&$0.62$&$0.7314$&$0.7362$&$0.8578$&$0.7734$&$0.6191$\\
&Time&$0.078$&$0.0459$&$0.0118$&$0.0295$&$0.3981$&$0.0237$\\
NNGC1\_dataset\_D1\_V1\_004&RMSE&$0.1629$&$0.1674$&$0.1687$&$0.2734$&$0.1724$&$0.1629$\\
(545,3)&SSE&$2.0459$&$2.1598$&$2.1962$&$5.9796$&$2.2944$&$2.0508$\\
&SSE/SST&$0.8326$&$0.8815$&$0.8915$&$0.9394$&$0.9204$&$0.8252$\\
&Time&$0.082$&$0.0557$&$0.0214$&$0.0586$&$0.7138$&$0.0571$\\
NNGC1\_dataset\_D1\_V1\_005&RMSE&$0.1274$&$0.1126$&$0.1186$&$0.1714$&$0.1234$&$0.1083$\\
(430,3)&SSE&$0.9772$&$0.7654$&$0.8482$&$1.7797$&$0.92$&$0.7111$\\
&SSE/SST&$0.2824$&$0.3117$&$0.3449$&$0.512$&$0.362$&$0.2897$\\
&Time&$0.0392$&$0.0672$&$0.0091$&$0.0348$&$0.2226$&$0.0451$\\
NNGC1\_dataset\_D1\_V1\_006&RMSE&$0.0895$&$0.0875$&$0.0899$&$0.0971$&$0.1055$&$0.0882$\\
(610,3)&SSE&$0.6848$&$0.6566$&$0.6914$&$0.8083$&$0.9524$&$0.6666$\\
&SSE/SST&$0.2711$&$0.264$&$0.2777$&$0.3217$&$0.3665$&$0.2679$\\
&Time&$0.0916$&$0.1702$&$0.0348$&$0.0736$&$0.7409$&$0.0795$\\
NNGC1\_dataset\_D1\_V1\_007&RMSE&$0.1274$&$0.1322$&$0.1269$&$0.2388$&$0.1363$&$0.125$\\
(610,3)&SSE&$1.3961$&$1.5047$&$1.3864$&$5.0833$&$1.5996$&$1.3414$\\
&SSE/SST&$0.6505$&$0.7312$&$0.6737$&$1.1004$&$0.7797$&$0.6503$\\
&Time&$0.0909$&$0.0792$&$0.0262$&$0.0874$&$0.9921$&$0.0452$\\
NNGC1\_dataset\_D1\_V1\_009&RMSE&$0.121$&$0.1229$&$0.1226$&$0.1365$&$0.1206$&$0.1179$\\
(540,3)&SSE&$1.11$&$1.1456$&$1.1389$&$1.4172$&$1.1032$&$1.0535$\\
&SSE/SST&$0.5401$&$0.5813$&$0.5778$&$0.686$&$0.5575$&$0.5354$\\
&Time&$0.0911$&$0.0912$&$0.0187$&$0.0264$&$0.7775$&$0.0823$\\
NNGC1\_dataset\_D1\_V1\_010&RMSE&$0.1432$&$0.1416$&$0.1423$&$0.2485$&$0.1438$&$0.1419$\\
(585,3)&SSE&$1.6854$&$1.6483$&$1.6631$&$5.0693$&$1.701$&$1.6537$\\
&SSE/SST&$0.7497$&$0.7493$&$0.7557$&$0.9095$&$0.7621$&$0.7506$\\
&Time&$0.1053$&$0.0349$&$0.0236$&$0.0637$&$1.0016$&$0.1168$\\
NNGC1\_dataset\_E1\_V1\_001&RMSE&$0.2218$&$0.2185$&$0.2219$&$0.3373$&$0.2201$&$0.2146$\\
(370,3)&SSE&$2.5633$&$2.5002$&$2.5672$&$5.9456$&$2.5204$&$2.4116$\\
&SSE/SST&$0.6469$&$0.6649$&$0.6846$&$0.8248$&$0.6739$&$0.6409$\\
&Time&$0.0428$&$0.0941$&$0.0096$&$0.0146$&$0.2745$&$0.1288$\\
NNGC1\_dataset\_E1\_V1\_008&RMSE&$0.1377$&$0.1297$&$0.1375$&$0.1713$&$0.1317$&$0.1353$\\
(740,3)&SSE&$1.9722$&$1.7544$&$1.9658$&$3.0604$&$1.8061$&$1.9128$\\
&SSE/SST&$0.7887$&$0.7201$&$0.8043$&$0.9674$&$0.744$&$0.788$\\
&Time&$0.149$&$0.0914$&$0.0342$&$0.0529$&$1.7184$&$0.0424$\\
NNGC1\_dataset\_E1\_V1\_009&RMSE&$0.1241$&$0.1127$&$0.1304$&$0.1662$&$0.1295$&$0.1327$\\
(740,3)&SSE&$1.6196$&$1.3371$&$1.7705$&$2.8728$&$1.7508$&$1.8535$\\
&SSE/SST&$0.71$&$0.5871$&$0.7752$&$0.9684$&$0.7633$&$0.8134$\\
&Time&$0.158$&$0.0843$&$0.0382$&$0.1167$&$1.6502$&$0.0427$\\
\hline
Average & RMSE & 0.1482 &	0.1382&	0.1416	&0.2078&	0.1449&	0.1371\\
 & SSE & 1.7197&	1.4947&	1.5721&	3.5412&	1.6311&	1.4904\\
 Average rank& RMSE & 3.55&	2.2&	3.6&	5.9&	3.9&	1.85\\
 & SSE & 3.5&	2.2	&3.6	&5.9	&3.9&	1.9\\
 &SSE/SST & 2.2	&2.5&	4.1&	5.9&	4.3&	2 \\
 \hline
\end{tabular}}
\end{table*}

\section{Application}{\label{sec:application}}
In this section, we present the results of numerical experiments conducted for the proposed LSTSVR-PI and the existing TSVR \cite{peng2010tsvr},  TWSVR \cite{khemchandani2016twsvr}, LSTSVR \cite{huang2016sparse}, TSVQR \cite{ye2023twin}, SVR+ \cite{vapnik2009new} on the time series real-world datasets \cite{derrac2015keel} to demonstrate the application of the proposed LSTSVR-PI model. 
We conducted random $70:30$ splitting for training and testing, respectively, and the ranges of different parameters considered are consistent with the experimental setup defined in Section \ref{sec:exp}. The datasets are normalized using min-max normalization shown in equation (\ref{eq:exp}).  Table \ref{tab:time_series_data} represents the performance of the models on the time series dataset along with the details of the datasets. The proposed LSTSVR-PI has the lowest average RMSE value of $0.1371$, the second lowest value is for TWSVR which is $0.1382$, followed by LSTSVR, SVR+, TSVR, and TSVQR with RMSE 0.1416, 0.1449, 0.1482, and 0.2078, respectively.  Using the optimal values corresponding to the least RMSE for each model across all datasets, we obtain the SSE and SSE/SST metrics. The average SSE is the smallest for the proposed LSTSVR-PI among all the models. As average RMSE or SSE values are not reliable always, we obtain the ranks of all the models across different datasets, considering different metrics. The average ranks corresponding to the RMSE values are 1.85, 2.2, 3.55, 3.6, 3.9, and 5.9 for proposed LSTSVR-PI, TWSVR, TSVR, LSTSVR, SVR+, and TSVQR, respectively. Considering SSE metric, the average rank assigned to the proposed LSTSVR-PI is $1.9$, which is the best among all the models.   \par
Employing the statistical tests i.e., Friedman and Nemenyi post hoc test as mentioned in subsection {\ref{subsec: statistical}}, we obtain $\chi_F^2=29.5571$ and $F_F=13.0126$ with degree of freedom $5$ and $(5,45)$, respectively. At $5\%$ level of significance, we have $F_{(5,45)}=2.422<F_F$, thus the null hypothesis of the Friedman test is rejected and models are not equivalent. Hence, the numerical results and statistical analysis prove the superiority of the proposed LSTSVR-PI. 


\section{Conclusion}{\label{sec:conclusion}}
In this article, we proposed least square twin support vector regression with privileged information (LSTSVR-PI) which considers the paradigm of learning using privileged information (LUPI) with the least square twin support vector regression. The proposed LUPI-based regression model constitutes the advantages of both. Most importantly, the LSTSVR-PI provides a learning process for the LSTSVR, wherein the extra privileged data is taken into account as a teacher in the training phase. Hence, the newly derived LSTSVR-PI bears an analogy to the teacher-student interaction seen in the human learning process. The proposed LSTSVR-PI considers the structural risk minimization principle by incorporating regularization terms corresponding to both the regressor function as well as the correction function, thus preventing overfitting. Consequently, the LSTSVR-PI has improved performance. Further LSTSVR-PI solves a system of linear equations, which adds to the efficiency. The nonlinear variant of the proposed LSTSVR-PI deals with the more intricate nonlinear relationships present in high-dimensional input data. We also established a generalization error bound based on the Rademacher complexity of the proposed model. Furthermore, to assess the performance of the proposed LSTSVR-PI, we conducted the experiments over artificially generated datasets and 21 real-world datasets. The numerical results and its statistical analysis demonstrate the superiority of the proposed model. As an application of the proposed LSTSVR-PI, we conducted the experiments on time-series datasets, which infers the proposed model is the best-performing one. However, it's worth noting that the proposed model has a large number of parameters to tune. This limitation can be overcome by incorporating different efficient tuning techniques to streamline the tuning process and improve the overall efficiency.
 

\section*{Acknowledgment}

We acknowledge Science and Engineering Research Board (SERB) under Mathematical Research Impact-Centric Support (MATRICS) scheme grant no. MTR/2021/000787 for supporting and funding the work. Ms. Anuradha Kumari (File no - 09/1022 (12437)/2021-EMR-I) would like to express her appreciation to the Council of Scientific and Industrial Research (CSIR) in New Delhi, India, for the financial assistance provided as fellowship. 
\bibliography{refs.bib}
\bibliographystyle{unsrtnat}
\end{document}